\documentclass{article}

\usepackage{microtype}
\usepackage{graphicx}
\usepackage{subfigure}
\usepackage{booktabs}

\usepackage{hyperref}

\usepackage[accepted]{icml2020}

\usepackage{multirow}
\usepackage{amssymb}
\usepackage{pifont}
\usepackage{adjustbox}
\usepackage{bm}

\usepackage[normalem]{ ulem }
\usepackage{soul}

\usepackage{arydshln}
\newcommand{\cmark}{\ding{51}}
\newcommand{\xmark}{\ding{55}}

\usepackage{stmaryrd}

\newcommand{\ie}{{i}.{e}.}
\newcommand{\eg}{{e}.{g}.}

\definecolor{red}{rgb}{1,0,0}

\usepackage{amsmath,amssymb}
\usepackage{bbm}
\usepackage{dsfont}
\usepackage{amsthm}
\newtheorem{theorem}{Theorem}

\newtheorem{corollary}{Corollary}[theorem]

\definecolor{blue}{rgb}{0,0,1}

\definecolor{lightgray}{rgb}{.83,.83,.83}
\definecolor{gray}{rgb}{.75,.75,.75}

\usepackage[small, skip=10pt]{caption}
\usepackage{caption} 
\icmltitlerunning{Continuous Domain Adaptation with Variational Domain-Agnostic Feature Replay}

\begin{document}

\twocolumn[
\icmltitle{Continuous Domain Adaptation with \\ Variational Domain-Agnostic Feature Replay}

\icmlsetsymbol{equal}{*}

\begin{icmlauthorlist}
\icmlauthor{Qicheng Lao}{imagia,mila}
\icmlauthor{Xiang Jiang}{imagia,dal}
\icmlauthor{Mohammad Havaei}{imagia}
\icmlauthor{Yoshua Bengio}{mila,cifar}
\end{icmlauthorlist}

\icmlaffiliation{imagia}{Imagia}
\icmlaffiliation{mila}{Montreal Institute for Learning Algorithms (MILA), University of Montreal}
\icmlaffiliation{cifar}{Canadian Institute for Advanced Research (CIFAR)}
\icmlaffiliation{dal}{Dalhousie University}
\icmlcorrespondingauthor{Qicheng Lao}{qicheng.lao@gmail.com}
\icmlcorrespondingauthor{Mohammad Havaei}{mohammad@imagia.com}

\icmlkeywords{Unsupervised Domain Adaptation, Non-stationary Environments, Continual Learning, Variational Inference, Replay}

\vskip 0.3in
]

\printAffiliationsAndNotice{}

\begin{abstract}
Learning in non-stationary environments is one of the biggest challenges in machine learning. Non-stationarity can be caused by either task drift, \ie, the drift in the conditional distribution of labels given the input data, or the domain drift, \ie, the drift in the marginal distribution of the input data. This paper aims to tackle this challenge in the context of continuous domain adaptation, where the model is required to learn new tasks adapted to new domains in a non-stationary environment while maintaining previously learned knowledge. To deal with both drifts, we propose variational domain-agnostic feature replay, an approach that is composed of three components: an \emph{inference module} that filters the input data into domain-agnostic representations, a \emph{generative module} that facilitates knowledge transfer, and a \emph{solver module} that applies the filtered and transferable knowledge to solve the queries. We address the two fundamental scenarios in continuous domain adaptation, demonstrating the effectiveness of our proposed approach for practical usage.
\end{abstract}

\section{Introduction} \label{sec_intro}
One of the biggest challenges in machine learning is to learn in non-stationary environments, in which the underlying data distribution (\ie, the joint distribution of the input data and labels $P(X,Y)$) changes over time, also referred to as \textit{concept drift} in previous literature~\cite{schlimmer1986incremental,widmer1996learning}. One source of the change can be from the drift in the conditional distribution of labels given the input data (\ie, $P(Y|X)$), often resulting from the change in the task definition, where the predictive function from the input space to label space may vary. Therefore, this drift in $P(Y|X)$ can be named as \textit{task drift}. Another source of the distribution change is the drift in the marginal distribution of the input data (\ie, $P(X)$), which we name as \textit{domain drift} here, with one additional assumption that $P(Y|X)$ remains the same, in lieu of the aforementioned task drift. The domain drift problem has been identified in many practical scenarios that tackle stream data, often with different terminologies, \eg, virtual concept drift~\cite{widmer1993effective,tsymbal2004problem}, feature change~\cite{gao2007general}, and many others summarized in the survey papers~\cite{gama2014survey,ditzler2015learning}.

Current research in continual learning~\cite{kirkpatrick2017overcoming,zenke2017continual,rebuffi2017icarl,shin2017continual,nguyen2017variational,schwarz2018progress} assumes a single non-stationary data stream, without considering the drifts between the training data and test data. In real-world applications, however, we normally have two streams of data arriving simultaneously, \ie, a training or \textit{support} stream and a test or \textit{query} stream, where both task drift and domain drift can be present across streams. For example, self-taught learning~\cite{raina2007self} can be viewed as a one-step adaptation of the task drift between the support data and query data, while most unsupervised domain adaptation algorithms~\cite{long2015learning,ganin2016domain,tzeng2017adversarial,hoffman2017cycada,saito2018maximum,pmlr-v97-zhang19i} resolve a single step of the domain drift without considering the stream data. There are works, however, that aim to tackle the domain drift for steam data, \ie, a stream of continuously evolving domains, but their proposed methodologies still lack the consideration of task drift in the stream~\cite{hoffman2014continuous,wulfmeier2018incremental}.

Aiming to tackle both drifts in non-stationary environments, this paper studies the problem of Continuous Domain Adaptation (ConDA) for real-life AI use cases. In ConDA, we assume data arriving in two streams (support and query), with the possibility of having both task drift and domain drift within and across the two streams. This is a practical scenario for many applications that use cloud services to ingest data. The goal for the AI model in the cloud is to continuously ingest the support data into some form of knowledge, and apply the learned knowledge to the queries that request predictive services. More specifically, the model needs to continuously accumulate knowledge from two perspectives: first, the knowledge should be \textit{filtered} to be domain-agnostic; second, the knowledge should be captured in a \textit{transferable} form that can be revisited at any time as needed, either to avoid catastrophic forgetting~\cite{mccloskey1989} by retaining competence on previously seen environments, or as part of domain-independent generic prior knowledge to help solve the current query of interest. Unlike many previous works in continual learning~\cite{rebuffi2017icarl,lopez2017gradient,shin2017continual} that address catastrophic forgetting by using unfiltered data (real or generated) to simulate a stationary environment, we emphasize on the high-level knowledge transfer for selective remembering. This is in line with the fact that although human brain has a huge amount of capacity, forgetting seems an evolutionarily correct mechanism. What we need is some form of abstract and transferable knowledge, such as textbook or dictionary, rather than data-level raw information.

To achieve the above, we propose a variational domain-agnostic feature replay approach, which is composed of three modules: (1) the \textit{inference module} that transforms the input data into filtered knowledge that is domain-agnostic; (2) the \textit{generative module} as a means to enable knowledge transfer by replaying the learned knowledge, similar to the high-level replay found in animal brains~\cite{skaggs1996replay}; and (3) the \textit{solver module} that applies the filtered and transferable knowledge to solve the queries. Intuitively, the synergy between the inference module and the solver module creates an information bottleneck, where the inference module minimizes the mutual information between the input and the domain-agnostic features while the solver module maximizes the mutual information between the domain-agnostic features and the labels.

We validate the proposed approach on two fundamental scenarios of continuous domain adaptation,~\ie, enforcing non-stationarity on the query stream either in the task space or in the domain space. Our experiments demonstrate the effectiveness of the proposed approach on addressing both task drift and domain drift. We also show the possibility of using generative domain-agnostic feature replay as a means of data augmentation towards better generalization.

\section{Continuous Domain Adaptation}
Let $S$ denote the support stream and $Q$ the query stream, both composed of sequentially arriving datasets, \ie, $S_i = \{ (\bm{x}_{S_i}^{(n)}, y_{S_i}^{(n)}) \}_{n=1, \dots, N_{S_i} }$ for $i \in [1, 2,...,T_S]$, and $Q_j=\{ \bm{x}_{Q_j}^{(n)} \}_{n=1,\dots,N_{Q_j}}$ for $j \in [1, 2,...,T_Q]$ (Figure~\ref{fig:conda}, top), the goal of continuous domain adaptation is to maintain the knowledge on $S$ and transfer to $Q$. Since the data streams can be acquired from any acquisition systems for any tasks, they may be subject to both task drift and domain drift. 

\begin{figure}[!t]
	\includegraphics[width=\columnwidth]{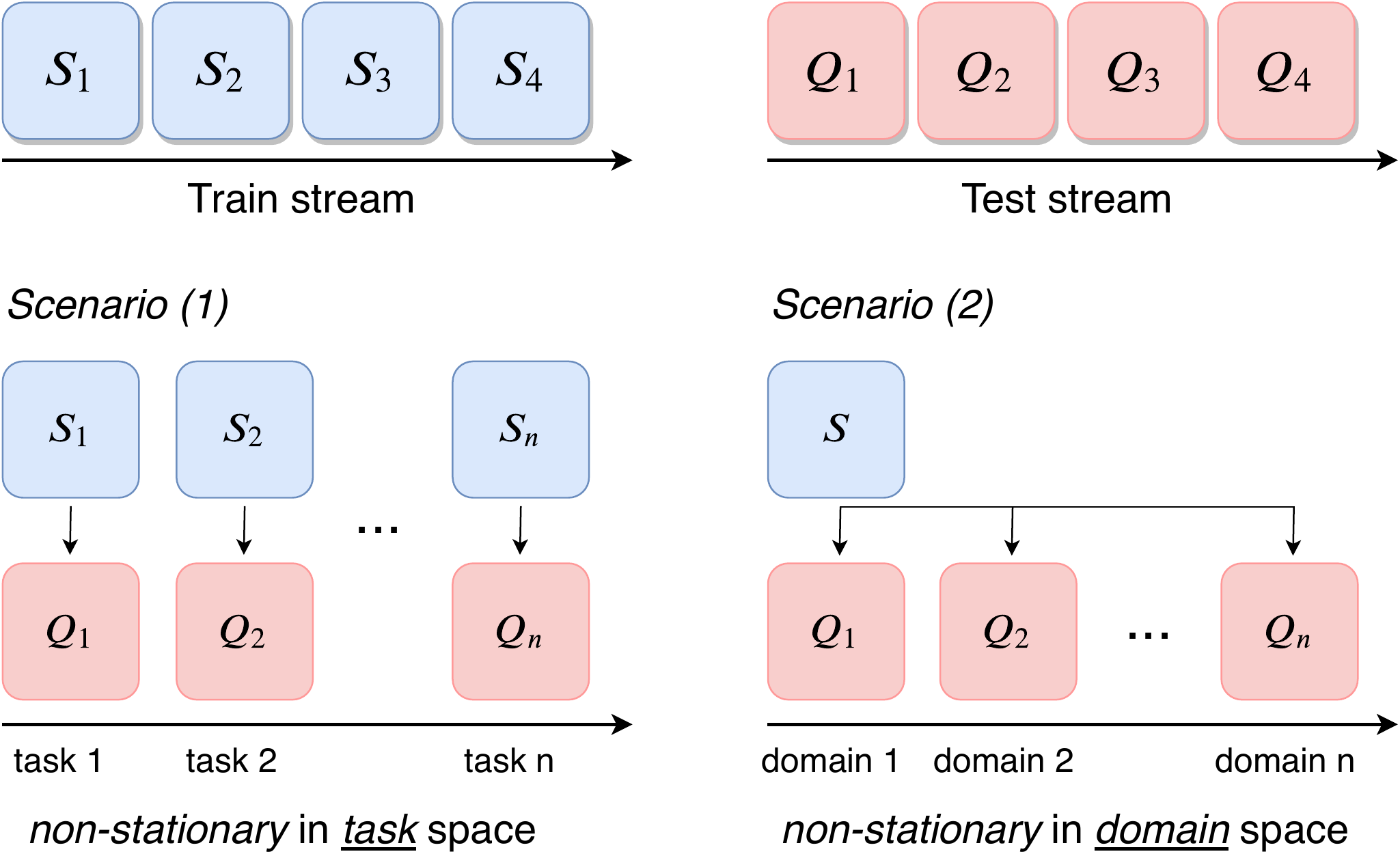}
	\vspace{-15pt}
	\caption{Streams of incoming data in the cloud. $S_i$ represents a training (\ie, support) set while $Q_j$ represents a test (\ie, query) set. We want to maintain our knowledge on $S$ and transfer to $Q$.}
	\label{fig:conda}
	\vspace{-15pt}
\end{figure}

The task drift can be illustrated in Scenario (1) (shown in Figure~\ref{fig:conda}, bottom left) with non-stationarity in the task space, where for the support stream $S$, the predictive mapping from the input data $X$ to class labels $Y$ can be continuously changing, \ie, $P_{S_i}(Y|X) \neq P_{S_j}(Y|X)$ for $i \neq j$, therefore resulting in different tasks, and the same non-stationarity also applies to the query stream $Q$. The challenge here is that the model needs to not only retain knowledge learned from past data for solving previous tasks (or queries), but also address the domain shift between the query and support streams since $P_{Q}(X) \neq P_{S}(X)$. The latter problem is well known as domain adaptation~\footnote{In this work, we focus on unsupervised domain adaptation.}~\cite{pan2009survey,quionero2009dataset}, and in order for the query data to be solvable, a typical assumption is that there exists an unobserved latent variable $X_0$, such that $P_{Q}(Y|X_0) = P_{S}(Y|X_0)$ ~\cite{quionero2009dataset}.

On the other hand, we can also assume non-stationarity in the domain space, where we have continuously arriving data in the query stream, each coming from a different domain (\ie, $P_{Q_i}(X) \neq P_{Q_j}(X)$ while $P_{Q_i}(Y|X) = P_{Q_j}(Y|X)$ for $i \neq j$) as shown in Scenario (2) (Figure~\ref{fig:conda}, bottom right). This scenario also represents a group of common use cases in practice, for example, in the healthcare sector, the sequentially arrived query data from different hospitals may be subject to domain drift due to the acquisition system settings, patient demographics, etc., however, the underlying predictive mechanism should remain unchanged. As a result, the model is required to be continuously updated to bridge new domain shift for the current query while maintaining the ability of solving previously seen queries. Note that, in both Scenario (1) and (2), we assume the usual restriction in continual learning that the data seen in previous environments is hidden, and we only have access to data in the current environment.

The above two scenarios compose the fundamental elements in building up most complex ConDA scenarios. Therefore, in this paper, we consider to solve these two fundamental scenarios. Moreover, in ConDA, we are interested in continuously solving unlabeled queries, whose solvability also depends on whether the corresponding support data is available, therefore, to make sure that the queries are actually solvable, we assume that for each $Q_i$, the corresponding $S_i$ (\ie, $S_i$ shares the same labeling function with $Q_i$) arrives earlier than $Q_i$. In practice, if $S_i$ arrives later, the adaptation of $Q_i$ can be held until $S_i$ is available. 

\begin{figure}[!t]
	\includegraphics[width=\columnwidth]{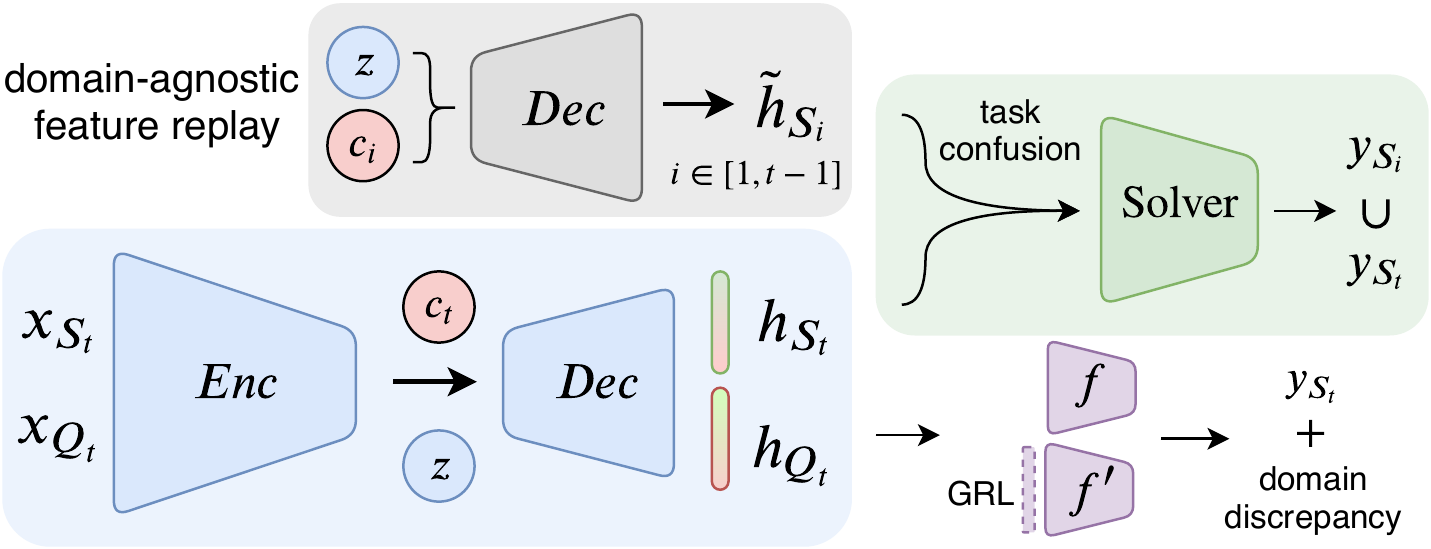}
	\vspace{-15pt}
	\caption{Decoupling continuous domain adaptation into three modules: Inference module (blue), Generative module (grey) and Solver module (green). The inference module infers the domain-agnostic features for the current task through variational inference; the generative module generates domain-agnostic feature samples from previous seen tasks; and the solver module is able to continuously solve all seen tasks.}
	\label{fig:model}
	\vspace{-15pt}
\end{figure}

\section{Approach}
Here, we propose to solve continuous domain adaptation by decoupling the problem into three modules (Figure~\ref{fig:model}): (1) \textit{Inference module}~(Section \ref{sec_inference}), which uses variational inference to train a domain-agnostic feature space for a given domain drift; (2) \textit{Generative module}~(Section \ref{sec_generation}), which allows us to sample from a previously learned domain-agnostic feature space. The sampled features can be used either as our filtered knowledge to reinforce the solver on remembering the knowledge required to solve the previous queries in a non-stationary environment, or as a data augmentation tool to augment the knowledge about the current query; and (3) \textit{Solver module}~(Section \ref{sec_solver}), which focuses only on the downstream task of interest, provided that the inferred or generated data is already domain-agnostic.

This decoupling also separates the concerns of domain drift from task drift in complex non-stationary environments, and brings up the possibility of transferring domain-agnostic knowledge among different environments, which to the best of our knowledge, has not yet been investigated in current domain adaptation research. As a first step, we will show later in our experiments that this is possible with our proposed approach.

\subsection{Variational domain-agnostic feature inference} \label{sec_inference}
In this module, we are given some labeled input data $\bm{x}_{S_t}$ from the support stream $S$, and unlabeled query data $\bm{x}_{Q_t}$ at time step $t$. We aim to map both support and query data into a shared stochastic feature space using a mapping function $\Psi$: $\bm{x} \rightarrow \bm{h}$, such that the following two conditions hold: (1) $\bm{h}$ maximally preserves the necessary information to predict the label $y$, and (2) the domain discrepancy between support and query on $\bm{h}$ is minimum. The first condition respects the current task in spite of task drift, while the second condition deals with domain drift. Note that for simplicity of notation, we ignore the subscripts for the random variables here.

To achieve this, we use variational inference~\cite{hoffman2013stochastic} through first encoding $\bm{x}$ into a latent variable $\bm{z}$, which is then decoded into $\bm{h}$. We also introduce a conditioning factor $c$ in order to enable the conditional generation for the generative module (Section~\ref{sec_generation}) in different environments, similar to~\cite{sohn2015learning}. The two conditions on $\bm{h}$ can be then formulated as:
\begin{equation} \label{eq_conditions}
\begin{split}
& \resizebox{1\linewidth}{!}{$\operatorname*{max}\limits_{\Psi, f} \mathbb{E}_{\bm{z} \sim q(\bm{z}|\bm{x}, c)} \log p(y|\bm{h})p(\bm{h}|\bm{z}, c) - KL(q(\bm{z}|\bm{x}, c) \vert\vert p(\bm{z})),$} \\
& \text{\quad s.t.\quad} d_{\mathcal{H}\Delta\mathcal{H}}(H_{S_t}, H_{Q_t}) \le \lambda_t,
\end{split}
\end{equation}
where $H_{S_t}$ and $H_{Q_t}$ are the marginal feature distributions for the support domain $S_t$ and query domain $Q_t$, respectively, and
\begin{multline}
\hspace{-0.8em} d_{\mathcal{H}\Delta\mathcal{H}}(H_{S_t}, H_{Q_t}) = \\
    \resizebox{1\linewidth}{!}{
    $2\sup\limits_{f,f'\in \mathcal{F}}\vert \mathbb{E}_{\bm{h} \sim H_{Q_t}} [f(\bm{h}) \neq f'(\bm{h})] - \mathbb{E}_{\bm{h} \sim H_{S_t}} [f(\bm{h}) \neq f'(\bm{h})] \vert$
    }\hspace{-1em}
\end{multline}
is the $\mathcal{H}\Delta\mathcal{H}$ divergence that measures domain discrepancy with $f$ and $f'$ denoting two hypotheses in the hypothesis space $\mathcal{F}$~\cite{ben2010theory}.

Solving~(\ref{eq_conditions}) is equivalent to minimizing the following objective function:
\begin{equation} 
    \mathcal{L} = \epsilon + \beta d_{\mathcal{H}\Delta\mathcal{H}},
\end{equation}
where $\epsilon$ denotes the error of satisfying the first condition in~(\ref{eq_conditions}), and $\beta$ is a Lagrange multiplier. Note that although derived from a new perspective, our objective function has a similar form as the upper bound of the target domain error~\footnote{equivalent to query domain error in this paper} in domain adaptation theory~\cite{ben2010theory}. In our case, we focus on constructing the domain-agnostic feature space that can be sampled from later on in the generative replay module~(Section \ref{sec_generation}). Following the works in adversarial domain adaptation~\cite{ganin2015unsupervised,tzeng2017adversarial,hoffman2017cycada,saito2018maximum,long2018conditional,pmlr-v97-zhang19i}, we also minimize the domain disparity discrepancy via a minimax optimization process:
\begin{equation} \label{eq_beta}
\begin{split}
    &\min_{\Psi, f} \quad \epsilon + \beta d_{\mathcal{H}\Delta\mathcal{H}}, \\
    &\max_{f'} \quad d_{\mathcal{H}\Delta\mathcal{H}}.
\end{split}
\end{equation}

\subsection{Generative domain-agnostic feature replay} \label{sec_generation}
Assuming a variational mapping from the input space to the domain-agnostic feature space has been learned, we can then use the decoder function $g$ to conditionally generate domain-agnostic features based on the conditioning factor c for each learned environment:
\begin{equation}
    \Tilde{\bm{h}} = g(\bm{z}, c), ~\bm{z} \sim \mathcal{N}(\bm{0},\bm{I}).
\end{equation}
Since $\Tilde{\bm{h}}$ is domain-agnostic, it represents the knowledge of interest filtered from the support data, and this knowledge can be transferred through the generative replay process, to further guide the training of the solver. Therefore, at each time step $t > 1$, even when the real data seen in previous environments (\ie, $i \in [1, t-1]$) is not available, we can still replay $\Tilde{\bm{h}_i}$ to address the catastrophic forgetting, \ie, to regularize the solver to remember how to solve previous queries. This is similar to generative replay, or pseudo-rehearsal~\cite{robins1995catastrophic}, which has been widely used as an effective approach to addressing catastrophic forgetting in continual learning. However, in most works~\cite{shin2017continual,wu2018memory}, the original input data $\bm{x}$ is replayed. While $p(\bm{x})$ is often difficult to approximate, especially when $\bm{x}$ is in high dimension, our feature replay can be viewed as a means for high-level knowledge transfer, \ie, \textit{filtered knowledge replay}, rather than the data-level replay, and the information filtration is guided by the inference module described above~(Section \ref{sec_inference}).

In addition to addressing catastrophic forgetting, our feature replay can also act as an alternative approach of transferring part of domain-independent generic prior knowledge among different queries. More specifically, the domain-agnostic feature $\Tilde{\bm{h}_i}$ resulting from solving previous query $Q_i$ can be used as augmented data in addition to the inferred ${\bm{h}_t}$ when solving the current query $Q_t$, given the assumption that $Q_t$ shares some similarity with $Q_i$. We discuss more details in Section~\ref{sec_exp_data_aug}.

\subsection{Solver module} \label{sec_solver}
Our solver is a unified model that continuously integrates knowledge filtered from seen data in the support stream, and is designed to solve all the seen queries. Since the solver operates on the domain-agnostic feature space, \ie, there is no domain shift between the support and query data in the feature space, it is therefore able to solve all the unlabeled queries in the query stream, once trained on the support stream. 

At time step $t$, our solver sees both the inferred features $\bm{h}_t$ from input data $\bm{x}_{S_t}$ in the support stream, and the generated features $\Tilde{\bm{h}_i}$ from previously learned snapshot decoder. Let $\bm{\theta}$ be the parameters of the solver, the objective can be given as:
\begin{align}
    & \min_{\bm{\theta}} \{ \mathbb{E}_{\bm{z} \sim q(\bm{z}|\bm{x}_{S_t}, c_t), \bm{h}_t \sim p(\bm{h}|\bm{z}, c_t)} \log p(y_{S_t}|\bm{h}_t) \nonumber \\
    & \quad\quad + \sum_{i=1}^{t-1} \mathbb{E}_{\bm{z} \sim p(\bm{z}), \bm{h}_i \sim p^*(\bm{h}|\bm{z}, c_i)} \log p(y_{S_i}|\bm{h}_i) \},
\end{align}
where $p^*$ denotes the snapshot decoders learned from previously seen environments.

Note that our solver is independent from the hypothesis classifier $f$ in the inference module. Although $f$ also aims to predict the class label $y_{S_t}$ given a domain-agnostic feature $\bm{h}_t$, it can not replace the role of the solver in solving previous queries, even with the replay of $\Tilde{\bm{h}_i}$. We will show later in our experiment (Figure~\ref{fig:acc_curves} (c)) that the feature replay of $\Tilde{\bm{h}_i}$ without the solver module can interfere with the adversarial learning when minimizing the domain disparity discrepancy for the current adaptation, thus resulting in impaired performance. Therefore, our solver module is indispensable as a way to confuse knowledge learned from all seen environments. We train the three modules end-to-end for continuous domain adaptation.

\subsection{Theoretical analysis} \label{sec_theorem}
Here, we analyze the theoretical guarantee for continuous domain adaptation.
\begin{theorem} \label{therorem_1}
Let $\lambda_i$ be the domain discrepancy of the marginal feature distributions $H_{S_i}$ and $H_{Q_i}$ measured by the $\mathcal{H}\Delta\mathcal{H}$ divergence, \ie, $\lambda_i = d_{\mathcal{H}\Delta\mathcal{H}}(H_{S_i}, H_{Q_i})$, and $P_{r}^{(i)}$, $P_{g}^{(i)}$ respectively denote the conditional distributions of labels $y_{S_i}$ given the real features $H_{S_i}$ and generated features $\Tilde{H}_{S_i}$, the total error of the query stream $\varepsilon_Q$ at time step $t$ is bounded by:
\begin{equation}
\begin{split}
    \varepsilon_Q \le \sum_{i=1}^t (\varepsilon_{S_i} + \lambda_i) + \sum_{i=1}^{t-1} KL(P_g^{(i)} \vert\vert P_r^{(i)}) + C^*,
\end{split}
\end{equation}
where $C^* = \min\limits_{\bm{\theta}} \sum_{i=1}^t (\varepsilon_{S_i} + \varepsilon_{Q_i})$ is the error of an optimal solver for both the support and query streams.
\end{theorem}

The proof is given in the supplementary material. The query stream error bound has different components, which can also be explained by our three different modules. For example, $\lambda_i$ represents how well our inference module has learned a domain-agnostic feature space; the $KL$ term measures the degree of the generative module approximating the real feature distribution, since we assume the previous real data is without access; and finally, $\varepsilon_{S_i}$ evaluates the solver's performance on the support data, and $C^*$ is the capacity of the solver in finding an optimal solution for both streams.

\begin{figure*}[!t]
	\includegraphics[width=\textwidth]{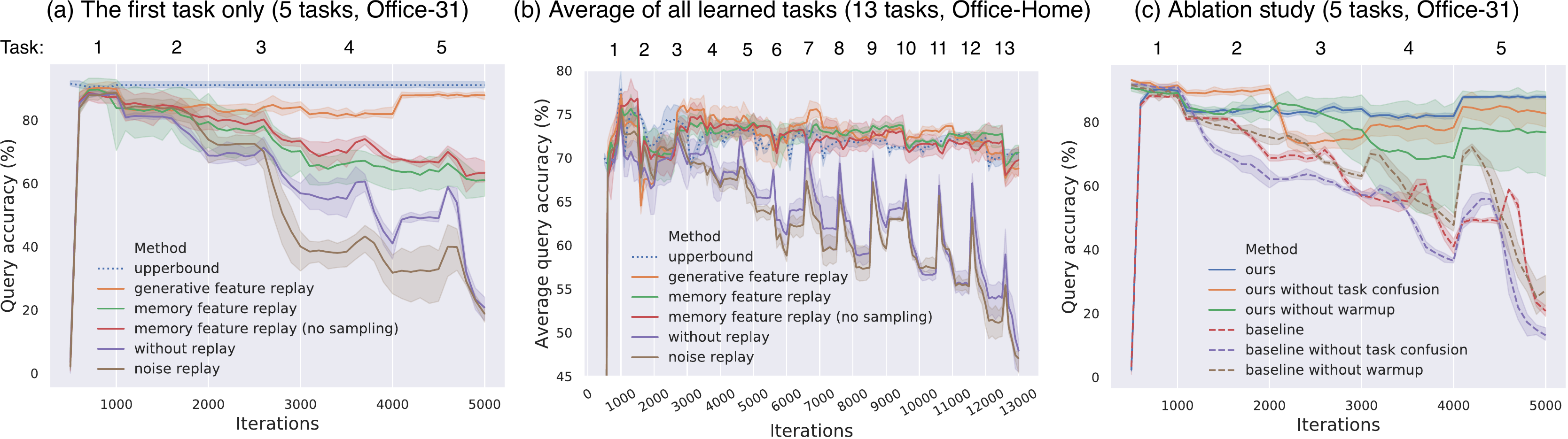}
	\vspace{-15pt}
	\caption{(a) The query accuracy on the first task during sequential training (D$\rightarrow$A, Office-31); (b) The average query accuracy of all learned tasks during sequential training (Ar$\rightarrow$Cl, Office-Home); (c) Ablation study on different components of our proposed approach. Baseline without warmup corresponds to baseline 2 in Table~\ref{tab_ablation}, and baseline without task confusion corresponds to baseline 3 in Table~\ref{tab_ablation}. }
	\label{fig:acc_curves}
	\vspace{-10pt}
\end{figure*}

\begin{figure}[!t]
	\includegraphics[width=\columnwidth]{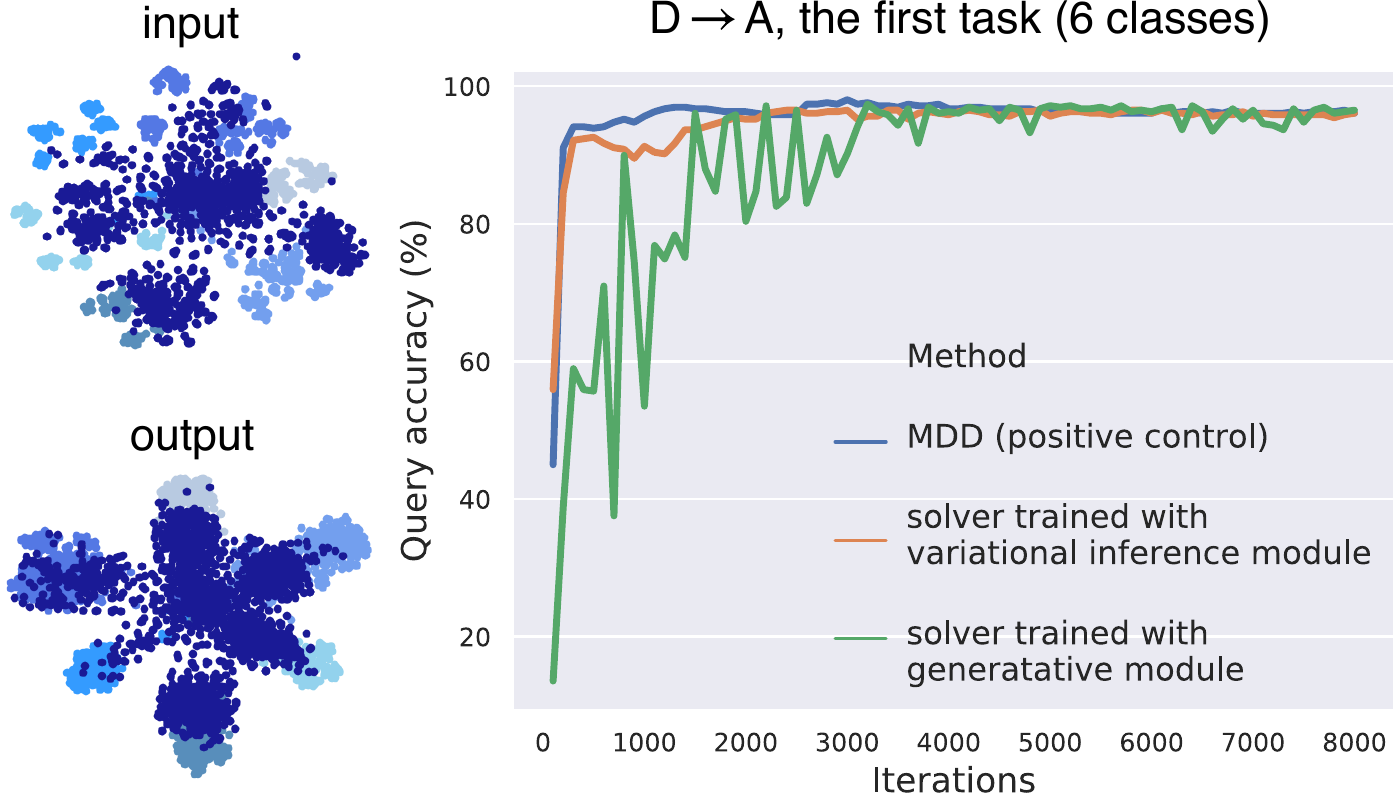}
	\caption{Comparison of features representing the support data (in \textit{light} colors) and query data (in \textit{dark} colors), before and after our inference module (left); The performance of the solver trained with generated features is comparable to that of the solver trained with real features (right).}
	\label{fig:feature_eval}
	\vspace{-10pt}
\end{figure}

\section{Experiments} \label{sec_exp}
We validate our proposed approach for continuous domain adaptation on two benchmark datasets. For scenario (1), we split the datasets into different tasks based on the class label to simulate the task drift in a non-stationary environment, and we consider one domain as the support stream with labels and another domain as the query stream without labels; For scenario (2), we choose one domain as the support stream and consider the remaining domains in the dataset as sequentially arriving queries in the query stream, so that the domain drift is present both within and across the two streams. We use margin disparity discrepancy (MDD) in our inference module for minimizing the domain disparity discrepancy, and also follow the same architecture choices as in~\cite{pmlr-v97-zhang19i}. More implementation details are provided in the supplementary material.

\paragraph{Dataset} Office-31~\cite{saenko2010adapting} has three domains: Amazon (A), DSLR (D) and Webcam (W), which in total contains 31 classes and 4,652 images. We split the dataset into 5 tasks, with 6 classes in the first four tasks and 7 classes in the last task (split details in Table~\ref{supp_tab_office31_split}, supplementary material). Office-Home~\cite{venkateswara2017deep} is a more challenging dataset that contains 65 classes and 15,500 images in four distinct domains: Artistic images (Ar), Clip art (Cl), Product images (Pr) and Real-World images (Rw). Similarly, we split the dataset into 13 tasks, each with 5 classes (split details in Table~\ref{supp_tab_officehome_split}, supplementary material).

\addtolength{\tabcolsep}{-3pt}  
\begin{table}[!t]
    \centering
    \caption{Components of our proposed approach.} 
    \begin{adjustbox}{width=\columnwidth}
    \begin{tabular}{@{}l *{6}{c}@{}}
        \toprule
        \multirow{2}{*}{Method} & \multicolumn{3}{c}{Replay} & \multirow{2}{*}{Task confusion} & \multirow{2}{*}{Warmup} & \multirow{2}{*}{Snapshot} \\
        \cmidrule(lr){2-4}
         & G & M & N & & & \\
        \midrule
        GFR* (Ours) & \cmark & & & \cmark & \cmark & \cmark \\
        Memory replay & & \cmark & & \cmark & \cmark & \cmark \\
        Noise replay & & & \cmark & \cmark & \cmark & \cmark \\
        Baseline 1 (Optimal) & \multicolumn{3}{c}{---} & \cmark & \cmark & \cmark \\
        Baseline 2 & \multicolumn{3}{c}{---} & \cmark & \xmark & \cmark \\
        Baseline 3 & \multicolumn{3}{c}{---} & \xmark & \xmark & \cmark \\
        Baseline 4 (Naive) & \multicolumn{3}{c}{---} & \xmark & \xmark & \xmark \\        
        \bottomrule
        \multicolumn{5}{l}{\footnotesize * short for generative feature replay} \\
      \end{tabular}
    \end{adjustbox}
    \label{tab_ablation}
    \vspace{-10pt}
\end{table}
\addtolength{\tabcolsep}{3pt}

\begin{figure*}[!t]
	\includegraphics[width=\textwidth]{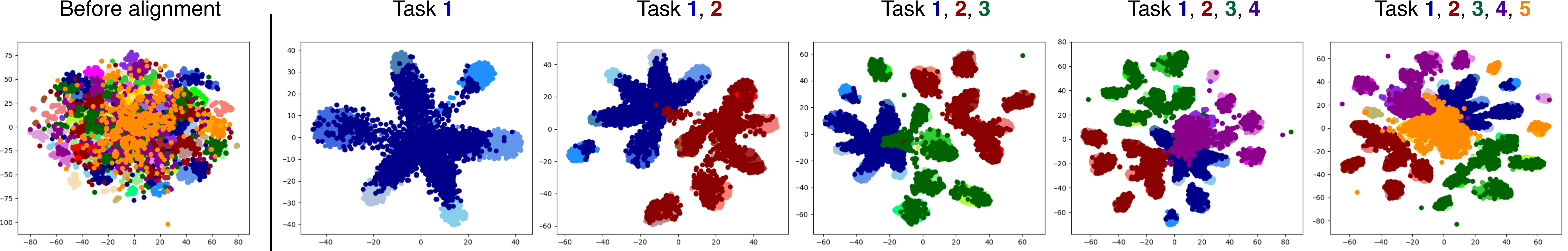}
	\vspace{-15pt}
	\caption{t-SNE visualizations of domain-agnostic features from both the support data (in \textit{light} colors) and query data (in \textit{dark} colors). The features are being continuously aligned given sequentially arriving tasks on Office-31.}
	\label{fig:tsne}
	\vspace{-5pt}
\end{figure*}

\addtolength{\tabcolsep}{-3pt}  
\begin{table*}[htbp]
    \centering
    \caption{Average query accuracy ($\%$) of all learned tasks on Office-Home.}
    \begin{adjustbox}{width=\textwidth}
      \begin{tabular}{c*{13}{c}}
        \toprule
        Method & Ar$\shortrightarrow$Cl & Ar$\shortrightarrow$Pr & Ar$\shortrightarrow$Rw & Cl$\shortrightarrow$Ar & Cl$\shortrightarrow$Pr & Cl$\shortrightarrow$Rw & Pr$\shortrightarrow$Ar & Pr$\shortrightarrow$Cl & Pr$\shortrightarrow$Rw & Rw$\shortrightarrow$Ar & Rw$\shortrightarrow$Cl & Rw$\shortrightarrow$Pr & \fcolorbox{gray}{lightgray}{Avg} \\
        \midrule
        joint training* & 54.9 & 73.7 & 77.8 & 60.0 & 71.4 & 71.8 & 61.2 & 53.6 & 78.1 & 72.5 & 60.2 & 82.3 & \fcolorbox{gray}{lightgray}{68.1} \\
        \midrule
        upper bound (task split) & 69.7 & 91.1 & 92.4 & 76.8 & 88.6 & 88.5 & 82.1 & 73.8 & 94.2 & 86.6 & 74.6 & 95.1 & \fcolorbox{gray}{lightgray}{84.5} \\
        \hdashline
        naive baseline$\dagger$ & 4.6 & 7.2 & 6.8 & 5.2 & 7.4 & 6.9 & 6.2 & 6.5 & 7.1 & 7.4 & 6.2 & 7.0 & \fcolorbox{gray}{lightgray}{\hspace{2.5pt}6.5\hspace{2.5pt}} \\
        optimal baseline$\dagger$ & 49.3 & 68.5 & 79.1 & 49.4 & 64.9 & 67.7 & 66.0 & 49.7 & 78.8 & 70.0 & 54.2 & 79.9 & \fcolorbox{gray}{lightgray}{64.8} \\
        ours & \textbf{69.0} & \textbf{87.4} & \textbf{88.3} & \textbf{71.6} & \textbf{87.2} & \textbf{87.2} & \textbf{76.2} & \textbf{70.0} & \textbf{92.2} & \textbf{84.5} & \textbf{72.7} & \textbf{93.8} & \fcolorbox{gray}{lightgray}{\textbf{81.7}} \\
        \bottomrule
        \multicolumn{5}{l}{\footnotesize * results cited from MDD~\cite{pmlr-v97-zhang19i}; $\dagger$ defined in Table~\ref{tab_ablation}.} \\
      \end{tabular}
    \end{adjustbox}
    \label{tab_average_home}
    \vspace{-12pt}
\end{table*}
\addtolength{\tabcolsep}{3pt}  

\subsection{Domain-agnostic feature evaluation}
We first evaluate the features from two perspectives: (1) whether the features can be domain-agnostic representations of filtered knowledge, and (2) whether the generated features can be a functional replacement of the real data for knowledge transfer, \ie, whether we can potentially use the proposed generative feature replay to facilitate the solver in remembering previously learned knowledge. To address the second question, we train a solver on the generated features, and evaluate it on the real features.

As shown in Figure~\ref{fig:feature_eval} (left), the output features through the inference module are aligned between the support data (in \textit{light} colors) and query data (in \textit{dark} colors), as compared to the features directly extracted from a ResNet~\cite{he2016deep} model pretrained on ImageNet~\cite{russakovsky2015imagenet}, indicating that the features are indeed domain-agnostic. In addition, we also demonstrate in Figure~\ref{fig:feature_eval} (right) that, the solver trained with generated features can predict the class labels as well as the solver trained with real features, although the convergence is slower with generated features. This suggests that feature replay is effective in approximating the real features for the downstream solver. It is also worth mentioning that the introduction of variational inference module does not impair the domain adaptation performance, with respect to MDD~\cite{pmlr-v97-zhang19i} as the positive control (Figure~\ref{fig:feature_eval}, right).

\subsection{Non-stationarity in tasks} \label{sec_task_drift}
Having shown the effectiveness of both inference and generative modules in Figure~\ref{fig:feature_eval}, we now use the proposed variational domain-agnostic feature replay to address Scenario (1), where we assume task drift in both streams and domain drift across streams (Figure~\ref{fig:conda}, bottom left). For the solver to be able to continuously solve the non-stationary queries, we replay the generated domain-agnostic features learned from previous tasks while learning the current task. This allows the solver to operate on both previous tasks and the current task simultaneously and thus function as a \textit{task confuser}, \ie, removes task boundaries. 

As shown in Figure~\ref{fig:acc_curves} (a), the generative feature replay helps the solver remember the first task as training progresses across tasks, whereas the solver suffers from catastrophic forgetting without replay (more results on other domains are shown in Figure~\ref{fig:task1_alldomains}, supplementary material). Similarly, the average query accuracy on all learned tasks can also be improved by the replay process (Figure~\ref{fig:acc_curves} (b)). Surprisingly, we also find that in some cases (\eg, Office-31 dataset shown in Figure~\ref{fig:acc_curves} (a)), the generative feature replay works better than the memory feature replay, where we store the features from real data in memory. One of the possible explanations could be that the generative feature distribution has learned the missing data points and acts as a regularizer in the feature space, which helps with overfitting especially when the training data examples are few (\eg, Office-31 dataset). This is also evidenced by the comparable performance between generative feature replay and memory feature replay on Office-Home dataset (Figure~\ref{fig:acc_curves} (b)), where more data examples are available. As a negative control, we also experiment with noise replay, where the features are replaced with random noises, and as expected, the solver suffers from catastrophic forgetting (Figure~\ref{fig:acc_curves} (a) and (b)). 

\vspace{-10pt}
\paragraph{Ablation study} We analyze the different model components and strategies used for our approach in Table~\ref{tab_ablation}, \eg, the solver module for the task confusion, the snapshot for the generative feature replay module, and the \textit{warmup} strategy. The warmup strategy is designed to first train the inference module independently for a few iterations, before integrating it with the training of the solver module in an end-to-end fashion. Figure~\ref{fig:acc_curves} (c) shows the results of the ablation study on both our approach and the baseline. It is shown in the figure that both the task confusion component and warmup strategy improve the performance of our approach while all baselines suffer severely from forgetting. 
\begin{table*}[htbp]
    \centering
    \caption{Average query accuracy ($\%$) of all learned domains on Office-Home.}
    \begin{adjustbox}{width=\textwidth}
      \begin{tabular}{c*{13}{c}}
        \toprule
        \multirow{3}{*}{Method} & & \multicolumn{5}{c}{Ascending order} & \multicolumn{5}{c}{Descending order} \\
        \cmidrule(lr){3-7} \cmidrule(lr){8-12}
        & Support stream: & Ar & Cl & Pr & Rw & \multirow{2}{*}{\fcolorbox{gray}{lightgray}{Avg}} & Ar & Cl & Pr & Rw & \multirow{2}{*}{\fcolorbox{gray}{lightgray}{Avg}} \\
        \cmidrule(lr){3-6} \cmidrule(lr){8-11}
        & Query stream: & Rw,Pr,Cl & Rw,Pr,Ar & Rw,Ar,Cl & Pr,Ar,Cl &  & Cl,Pr,Rw & Ar,Pr,Rw & Cl,Ar,Rw & Cl,Ar,Pr &  \\
        \midrule
        
        \multicolumn{2}{c}{upper bound (domain split)} & 65.45 & 61.98 & 60.04 & 67.18 & \fcolorbox{gray}{lightgray}{63.66} & 64.88 & 60.04 & 58.24 & 67.46 & \fcolorbox{gray}{lightgray}{62.66} \\
        \hdashline
        \multicolumn{2}{c}{without replay} & 18.41 & 18.86 & 16.90 & 20.31 & \fcolorbox{gray}{lightgray}{18.62} & 25.84 & 21.31 & 25.21 & 28.42 & \fcolorbox{gray}{lightgray}{25.20} \\
        \multicolumn{2}{c}{without encoder snapshot} & 56.12 & 56.69 & 54.07 & 57.91 & \fcolorbox{gray}{lightgray}{56.20} & \textbf{62.31} & 57.74 & \textbf{56.92} & 62.58 & \fcolorbox{gray}{lightgray}{\textbf{59.89}} \\ 
        \multicolumn{2}{c}{ours} & \textbf{62.00} & \textbf{59.17} & \textbf{57.09} & \textbf{62.63} & \fcolorbox{gray}{lightgray}{\textbf{60.22}} & 60.00 & \textbf{57.99} & {54.95} & \textbf{65.05} & \fcolorbox{gray}{lightgray}{59.50}  \\
        \bottomrule
      \end{tabular}
    \end{adjustbox}
    \label{tab_domains}
    \vspace{-10pt}
\end{table*}

Figure~\ref{fig:tsne} shows the t-SNE visualizations of features from both the support data (in \textit{light} colors) and query data (in \textit{dark} colors) at each task step, where the class space is gradually expanding. The class-wise alignment between the support stream and query stream guarantees the solver's performance on the query stream, since the solver is only trained with the support data in our approach. We summarize the performance comparisons of the average query accuracy between our approach and multiple baselines in Table~\ref{tab_average_home} (Office-Home) and Table~\ref{tab_average_31} (Office-31, supplementary material). The superior performance of our proposed approach demonstrates its effectiveness in addressing both task drift and domain drift that are present in Scenario (1).

\subsection{Non-stationarity in domains} \label{sec_domain_drift}
In this section, we address Scenario (2), in which we assume a single domain in the support stream, and sequentially arriving queries from different domains in the query stream as shown in Figure~\ref{fig:conda} (bottom right). As such, the domain drift exists both within and across streams. We perform experiments on Office-Home dataset by selecting one domain as the support data and the remaining three domains as queries in the query stream ordered by the adaptation difficulty level~\footnote{based on query accuracy marginalized on the support data}, either ascending (\ie, Rw, Pr, Ar, Cl) or descending (\ie, Cl, Ar, Pr, Rw). For example, if domain Ar is chosen as the support data, the adaptation of the query stream can be written as Ar$\rightarrow$Rw, Pr, Cl in ascending order, and Ar$\rightarrow$Cl, Pr, Rw in descending order. We evaluate the effectiveness of generative feature replay in the worse case scenario of catastrophic forgetting, where the solver deteriorates into a complete forgetting. To simulate this, we train the solver from scratch for each query domain, and investigate the effect of generative domain-agnostic feature replay on the overall performance of all seen domains. 

Figure~\ref{fig:acc_domains} shows the curves of the average query accuracy of all learned domains, and it demonstrates that with all listed permutations, generative feature replay facilitates the solver to generalize to all previously seen domains when the data for previous query domains is without access. This again shows the ability of generative feature replay in a de novo transfer of high-level knowledge to the solver without relying on the example-level experiences. Table~\ref{tab_domains} compares the average query accuracy of our approach to that of different baselines, where we can see the dramatic improvements with replay. We also find that the knowledge transfer can be further facilitated by keeping snapshot of the encoder from the inference module, especially when the query stream is in an ascending order (\ie, from easy to hard), suggesting that easy queries are more vulnerable to forgetting.

\begin{figure}[!t]
	\includegraphics[width=\columnwidth]{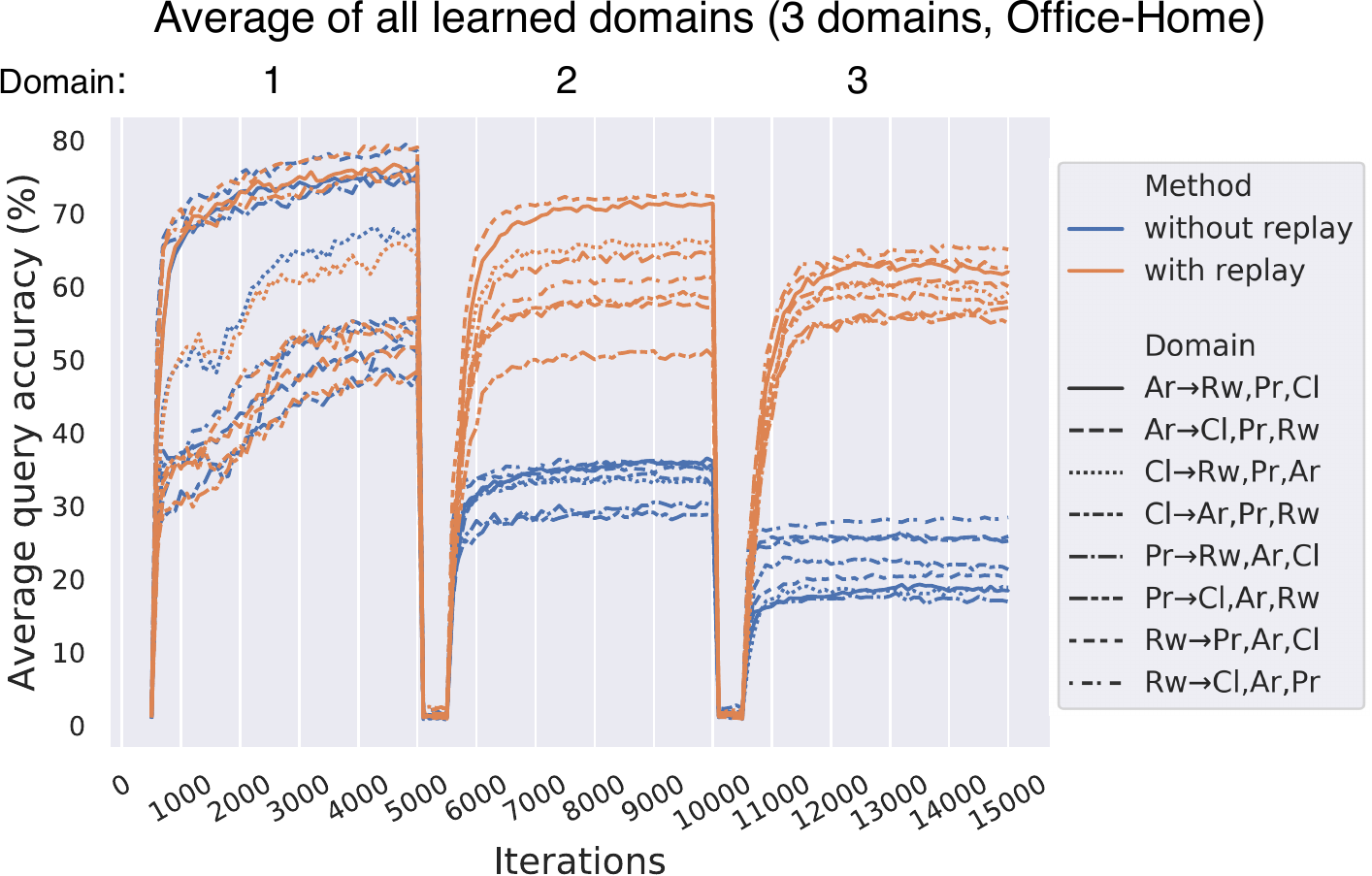}
	\caption{The average query accuracy of all learned domains during sequential training (Office-Home).}
	\label{fig:acc_domains}
	\vspace{-10pt}
\end{figure}

\subsection{Generative feature replay for data augmentation} \label{sec_exp_data_aug}
Note that in Scenario (2), as training progresses in the query stream, the solver module eventually captures generalized features that are agnostic to all seen domains. This is analogous to learning class-specific features that lay in the intersection of different domains as shown in Figure~\ref{fig:domains_aug} (a). In our proposed approach, we learn domain-agnostic features ${\bm{h}_i}$ between the support domain $S$ and the query domain $Q_i$ at time step $i$, \eg, the intersection of support domain and query 1 in Figure~\ref{fig:domains_aug} (a) (left). The generative feature replay module learned at time step $i$ can be used when solving a subsequent query domain $Q_t$ ($t > i$). However, whether replaying the generated features $\Tilde{\bm{h}_i}$ can be beneficial for solving $Q_t$ depends on the similarity between $Q_i$ and $Q_t$. For example, as illustrated in Figure~\ref{fig:domains_aug} (a), query~1 shares more similarity with query 2 than query 3, therefore the knowledge learned from solving query~1 would generally be more transferable to query 2.

\begin{figure}[!t]
	\includegraphics[width=\columnwidth]{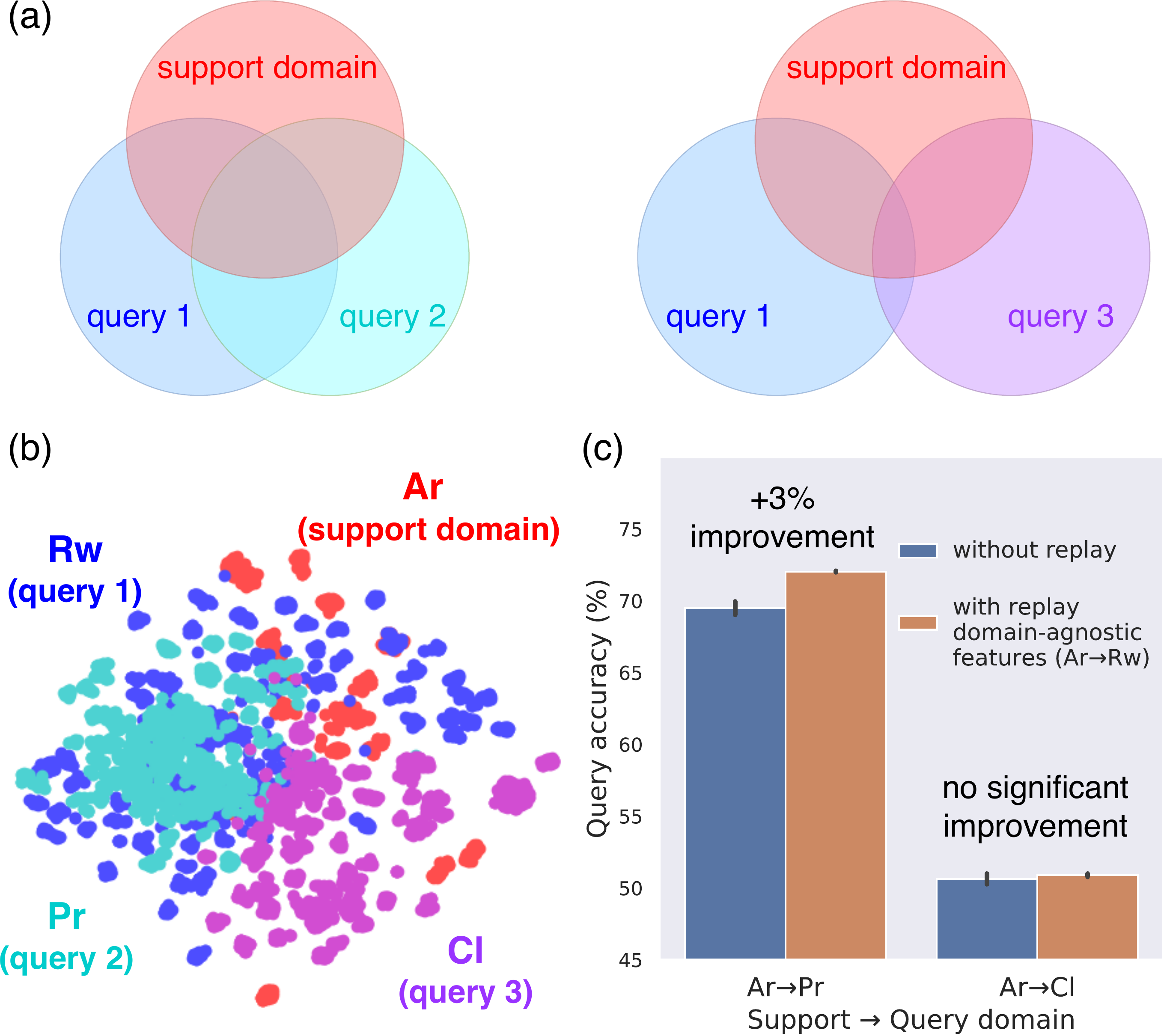}
	\vspace{-15pt}
	\caption{(a) Examples of possible relation among domains; (b) t-SNE visualization of features from the four different domains in Office-Home dataset (showing the first class only); (c) Generative replay of the features learned from Ar$\rightarrow$Rw improves Ar$\rightarrow$Pr, but not Ar$\rightarrow$Cl.}
	\label{fig:domains_aug}
	\vspace{-10pt}
\end{figure}

To illustrate this, we first visualize the features of the first class (Table~\ref{supp_tab_officehome_split}, supplementary material) from the four different domains in Office-Home dataset as an estimation of domain relation. The features are extracted using a ResNet-50 model pretrained on ImageNet. As seen from the t-SNE plot in Figure~\ref{fig:domains_aug} (b), domain Rw shares more overlap with domain Pr as compared to domain Cl. Correspondingly, we find in our experiment that, given domain Ar as the support domain, the generative replay of features learned from solving Ar$\rightarrow$Rw improves Ar$\rightarrow$Pr by around 3$\%$ in the query accuracy, while no significant improvement is observed for Ar$\rightarrow$Cl (Figure~\ref{fig:domains_aug} (c)). However, generally speaking, the improvement is found to be a more common phenomenon, for example, Cl$\rightarrow$Rw improves Cl$\rightarrow$Pr by 1.35$\%$, Ar$\rightarrow$Cl improves Ar$\rightarrow$Pr by 3.45$\%$, and Pr$\rightarrow$Cl improves Pr$\rightarrow$Ar by 1.89$\%$ (results are shown in Figure \ref{fig:domains_aug_suppl}, supplementary material). Given the observed improvements, it is possible that our generative feature replay, by providing more augmented feature samples that are domain-agnostic, imposes an additional regularization to constrain the solver in capturing more generalized features. However, this is based on the assumption that the solver has a fixed amount of capacity.

\section{Related Work}
\paragraph{Continuous domain adaptation}
The problem of continuous domain adaptation has been studied before but in different contexts with different emphases. For example, \cite{mancini2019adagraph} attempt to solve a specific scenario in continuous domain adaptation, where no target data is available, but with metadata provided for all domains; \cite{gong2019dlow} propose to bridge two domains by generating a continuous flow of intermediate domains between the two original domains; ~\cite{hoffman2014continuous,wulfmeier2018incremental} present continuous domain adaptation with the emphasis to generalize on a transitioning target domain. Closely related to our Scenario (2) is the recent work of~\cite{bobu2018adapting}, where they also aim to address catastrophic forgetting, but with an implicit assumption that the domain drift follows a specific pattern, \ie, induced by gradually changing weather or lighting condition, which is a reasonable assumption in applications such as autonomous driving. We focus on more general use cases for solving any arriving queries in the cloud without imposing extra constraints on the relationship among the queries.

\paragraph{Variational information bottleneck}
Our work is also related to variational information bottleneck~\cite{alemi2016deep}, in the sense that we address the domain drift across streams via a variational inference that can be viewed as maximizing the mutual information between the domain-agnostic features and labels, while minimizing the mutual information between the input data and domain-agnostic features. A concurrent work~\cite{song2019improving} adopts the idea of variational information bottleneck for domain adaptation, where the one-step domain adaptation performance is shown to be improved. Similarly, \cite{luo2019significance} show the integration of information bottleneck improves domain adaptive segmentation task. In our approach, we constrain the bottleneck on the decoded feature rather than directly on the latent code, and require no additional regularization on the query (target) data as in~\cite{song2019improving}.

Variational autoencoder~\cite{kingma2013auto} has also been extensively exploited in domain adaptation to learn disentangled representations for better adaptation performance, where different types of latent variables are proposed to better capture the variations in the dataset (\eg, domain-relevant and class-relevant information), and the reconstruction is either on the image level~\cite{ilse2019diva,cai2019learning} or feature level~\cite{peng2019domain}. In our variational inference module, the domain-agnostic features are learned through supervision from labels instead of reconstruction.

\vspace{-2pt}
\paragraph{Replay in continual learning}
Replay has been widely used as an effective approach to addressing catastrophic forgetting in the continual learning research, such as example replay~\cite{rebuffi2017icarl}, deep generative replay~\cite{shin2017continual, wu2018memory}, and experience replay~\cite{rolnick2019experience}. However, in these approaches, the generative process is unfiltered and operates on the data level, while our generative process is domain-agnostic and operates on the abstract feature level.  A concurrent work~\cite{pellegrini2019latent} that uses latent replay is closely related to our feature replay, both emphasizing on the high-level knowledge transfer, however, their latent replay stands for the replay of the activation volumes in some of the intermediate layers without stochasticity.

\section{Conclusion}
In this paper, we tackle the challenge of learning in non-stationary environments in the context of continuous domain adaptation, where we have two streams of data in the cloud (\ie, support and query steam) that can be subject to both task drift and domain drift, within and across streams. We present two fundamental scenarios for continuous domain adaptation with the presence of across-stream domain drift, by assuming either task drift or domain drift in both streams. To address both drifts, we propose a variational domain-agnostic feature replay approach, which allows the model in the cloud to continuously accumulate the filtered and transferable knowledge for solving all queries. We demonstrate the effectiveness of the proposed approach on the two fundamental scenarios in continuous domain adaptation.

\bibliography{paper}
\bibliographystyle{icml2020}

\clearpage
\section{Supplementary Material}
\subsection{Proof of Theorem~\ref{therorem_1})}
In this subsection, we give proof of Theorem~\ref{therorem_1} presented in Section~\ref{sec_theorem}, which analyzes the theoretical guarantee for continuous domain adaptation.

\begin{theorem}~\cite{ben2010theory}
Given a source domain and a target domain, with the input data distributions $X_{\text{source}}$ and $X_{\text{target}}$, we have the target domain error bounded by:
\begin{equation}
    \varepsilon_{\text{target}} \le \varepsilon_{\text{source}} + d_{\mathcal{H}\Delta\mathcal{H}}(X_{\text{source}}, X_{\text{target}}) + C^*,
\end{equation}
where $d_{\mathcal{H}\Delta\mathcal{H}}(X_{\text{source}}, X_{\text{target}})$ is the $\mathcal{H}\Delta\mathcal{H}$ divergence that measures domain discrepancy between the source and target domain, and $C^*$ is the error of an optimal classifier for both the source and target domains.
\end{theorem}

\begin{corollary} \label{corollary_1}
For each time step $i$ in the support stream $S$ and query stream $Q$, let $\lambda_i = d_{\mathcal{H}\Delta\mathcal{H}}(H_{S_i}, H_{Q_i})$ be the domain discrepancy of the feature distributions $H_{S_i}$ and $H_{Q_i}$, the error of the query domain ${Q_{i}}$ is bounded by:
\begin{equation}
    \varepsilon_{Q_{i}} \le \varepsilon_{S_{i}} + \lambda_i + C_i^*,
\end{equation}
where $C_i^* = \min\limits_{\bm{\theta}} (\varepsilon_{S_i} + \varepsilon_{Q_i})$ is the error of an optimal solver for both the support domain $S_i$ and the query domain $Q_i$.
\end{corollary}

With the setup introduced in Corollary~\ref{corollary_1}, we now prove Theorem~\ref{therorem_1}.

\setcounter{theorem}{0}
\begin{theorem} (Theorem~\ref{therorem_1} in Section~\ref{sec_theorem})
Let $\lambda_i$ be the domain discrepancy of the marginal feature distributions $H_{S_i}$ and $H_{Q_i}$ measured by the $\mathcal{H}\Delta\mathcal{H}$ divergence, \ie, $\lambda_i = d_{\mathcal{H}\Delta\mathcal{H}}(H_{S_i}, H_{Q_i})$, and $P_{r}^{(i)}$, $P_{g}^{(i)}$ respectively denote the conditional distributions of labels $y_{S_i}$ given the real features $H_{S_i}$ and generated features $\Tilde{H}_{S_i}$, the total error of the query stream $\varepsilon_Q$ at time step $t$ is bounded by:
\begin{equation}
\begin{split}
    \varepsilon_Q \le \sum_{i=1}^t (\varepsilon_{S_i} + \lambda_i) + \sum_{i=1}^{t-1} KL(P_g^{(i)} \vert\vert P_r^{(i)}) + C^*,
\end{split}
\end{equation}
where $C^* = \min\limits_{\bm{\theta}} \sum_{i=1}^t (\varepsilon_{S_i} + \varepsilon_{Q_i})$ is the error of an optimal solver for both the support and query streams.
\end{theorem}

\begin{proof}
At time step $t$, the error of previous query domains $Q_i$ for $i \in [1, t-1]$ is estimated by $\hat{\varepsilon}_{Q_i}$, since we use $\Tilde{H}_{S_i}$ to approximate $H_{S_i}$ during the training of the solver. The total error of the query stream is:
\begin{equation}\label{epsilon_Q}
\begin{split}
    \varepsilon_Q &= \varepsilon_{Q_t} + \sum_{i=1}^{t-1} \hat{\varepsilon}_{Q_i},\\
    & \le \varepsilon_{S_t} + \lambda_t + \sum_{i=1}^{t-1} (\hat{\varepsilon}_{S_i} + \lambda_i) + C^*.\\
\end{split}
\end{equation}

On the other hand, for each support domain $S_i$, the $KL$ divergence between the real and generated conditional distributions of labels $y_{S_i}$ given the features, \ie, $KL(P_g^{(i)}(y|\bm{h}) \vert\vert P_r^{(i)}(y|\bm{h}))$, can be interpreted as:
\begin{equation}
\begin{split}
    KL(P_g^{(i)}(y|\bm{h}) \vert\vert P_r^{(i)}(y|\bm{h})) \\
    &\hspace{-75pt}= \sum_y P_g^{(i)}(y|\bm{h}) \log \frac{P_g^{(i)}(y|\bm{h})}{P_r^{(i)}(y|\bm{h})} \\
    &\hspace{-75pt}= \mathbb{E}[\log \frac{P_g^{(i)}(y|\bm{h})}{P_r^{(i)}(y|\bm{h})}] \\
    &\hspace{-75pt}= \mathbb{E}[\log P_g^{(i)}(y|\bm{h})] - \mathbb{E}[\log P_r^{(i)}(y|\bm{h})] \\
    &\hspace{-75pt}= \hat{\varepsilon}_{S_i} - \varepsilon_{S_i}.
\end{split}
\end{equation}
Therefore, the error for each support domain $S_i$ at the time step $t$ is estimated by:
\begin{equation}\label{epsilon_S}
\begin{split}
    \hat{\varepsilon}_{S_i} = \varepsilon_{S_i} + KL(P_g^{(i)} \vert\vert P_r^{(i)}).
\end{split}
\end{equation}

Combining \ref{epsilon_Q} and \ref{epsilon_S}, the total error of the query stream $\varepsilon_Q$ at time step $t$ is:
\begin{equation}
\begin{split}
    \varepsilon_Q &\le \varepsilon_{S_t} + \lambda_t + \sum_{i=1}^{t-1} (\hat{\varepsilon}_{S_i} + \lambda_i) + C^* \\
    &= \varepsilon_{S_t} + \lambda_t + \sum_{i=1}^{t-1} ({\varepsilon}_{S_i} + KL(P_g^{(i)} \vert\vert P_r^{(i)}) + \lambda_i) + C^* \\
    &= \sum_{i=1}^t (\varepsilon_{S_i} + \lambda_i) + \sum_{i=1}^{t-1} KL(P_g^{(i)} \vert\vert P_r^{(i)}) + C^*.
\end{split}
\end{equation}

\end{proof}

\subsection{Implementation details}
\subsubsection{Dataset task split}
The two benchmark datasets are split into multiple tasks based on the class labels to simulate the task drift in non-stationary environments. Office-31~\cite{saenko2010adapting} has three domains: Amazon (A), DSLR (D) and Webcam (W), which in total contains 31 classes and 4,652 images. We split the dataset into 5 tasks, with 6 classes in the first four tasks and 7 classes in the last task (split details in Table~\ref{supp_tab_office31_split}). Office-Home~\cite{venkateswara2017deep} is a more challenging dataset that contains 65 classes and 15,500 images in four distinct domains: Artistic images (Ar), Clip art (Cl), Product images (Pr) and Real-World images (Rw). Similarly, we split the dataset into 13 tasks, each with 5 classes (split details in Table~\ref{supp_tab_officehome_split}).

\begin{table*}[htbp]
    \centering
    \caption{Office-31 task split.}
      \begin{tabular}{r*{10}{c}}
        \toprule
        Task & Classes (number of classes) & Label range \\
        \midrule
        1 & back pack, bike, bike helmet, bookcase, bottle, calculator (6) & 0-5 \\
        2 & desk chair, desk lamp, desktop computer, file cabinet, headphones, keyboard (6) & 6-11 \\
        3 & laptop computer, letter tray, mobile phone, monitor, mouse, mug (6) & 12-17 \\
        4 & paper notebook, pen, phone, printer, projector, punchers (6) & 18-23 \\
        5 & ring binder, ruler, scissors, speaker, stapler, tape dispenser, trash can (7) & 24-30 \\
        \bottomrule
      \end{tabular}
    \label{supp_tab_office31_split}
\end{table*}

\begin{table*}[htbp]
    \centering
    \caption{Office-Home task split.}
      \begin{tabular}{r*{10}{c}}
        \toprule
        Task & Classes (number of classes) & Label range \\
        \midrule
        1 & Drill, Exit Sign, Bottle, Glasses, Computer (5) & 0-4 \\
        2 & File Cabinet, Shelf, Toys, Sink, Laptop (5) & 5-9 \\
        3 & Kettle, Folder, Keyboard, Flipflops, Pencil (5) & 10-14 \\
        4 & Bed, Hammer, Toothbrush, Couch, Bike (5) & 15-19 \\
        5 & Postit Notes, Mug, Webcam, Desk Lamp, Telephone (5) & 20-24 \\
        6 & Helmet, Mouse, Pen, Monitor, Mop (5) & 25-29 \\
        7 & Sneakers, Notebook, Backpack, Alarm Clock, Push Pin (5) & 30-34 \\
        8 & Paper Clip, Batteries, Radio, Fan, Ruler (5) & 35-39 \\
        9 & Pan, Screwdriver, Trash Can, Printer, Speaker (5) & 40-44 \\
        10 & Eraser, Bucket, Chair, Calendar, Calculator (5) & 45-49 \\
        11 & Flowers, Lamp Shade, Spoon, Candles, Clipboards (5) & 50-54 \\
        12 & Scissors, TV, Curtains, Fork, Soda (5) & 55-59 \\
        13 & Table, Knives, Oven, Refrigerator, Marker (5) & 60-64 \\
        \bottomrule
      \end{tabular}
      \vspace{4pt}
    \label{supp_tab_officehome_split}
\end{table*}

\subsubsection{Model architectures}
We adopt ResNet~\cite{he2016deep} models pretrained on ImageNet~\cite{russakovsky2015imagenet} as part of our encoder, \eg, ResNet-34 for the Office-31 dataset and ResNet-50 for the Office-Home dataset. The extracted features are used to infer the latent code with one additional linear layer. Our decoder consists of a linear layer, a ReLU layer and a Batch Norm layer. The dimension of the output domain-agnostic features is 1024 for Office-31 and 2048 for Office-Home. The solver module is a two-layer neural network with the width of 1024. The parameters of the ResNet model used in Scenario (2) is fine-tuned during training in order to learn better class representations, since more classes are involved within a single task, as compared to Scenario (1), where the extracted features directly from the ResNet model are sufficient enough as class representations.

We use margin disparity discrepancy (MDD) in our inference module for minimizing the domain disparity discrepancy, and also follow the same architecture choices as in~\cite{pmlr-v97-zhang19i} for the two classifiers $f$ and $f'$, \ie, two-layer neural networks.

\subsubsection{Optimization}
Two optimizers are used: Adam optimizer for the inference module with the learning rate 1e-4, and SGD optimizer for the two classifiers ($f$, $f'$) and the solver module, with nesterov momentum 0.9 and weight decay 5e-4. The initial learning rate of the SGD optimizer is set to 4e-4 for Scenario~(1) and 4e-3 for Scenario (2). 

We use the gradient reversal strategy~\cite{ganin2016domain} for the minmax optimization (Eq.~\ref{eq_beta}), and the training scheduler for the coefficient in the gradient reversal layer is defined by:
\begin{equation}
    {coeff} = 2.0 * \frac{0.3}{1.0 + \exp{(-\frac{i}{2000})}} - 0.3,
\end{equation}
where $i$ is the iteration step number. The Lagrange multiplier $\beta$ in Eq.~\ref{eq_beta} is set to 1. 

We train the three modules (the inference module, generative module and solver module) end-to-end with 1000 iteration steps for each task for Office-31, and 5000 iteration steps for Office-Home. We warmup the inference module for 500 iteration steps.

\begin{figure*}[!t]
	\includegraphics[width=\textwidth]{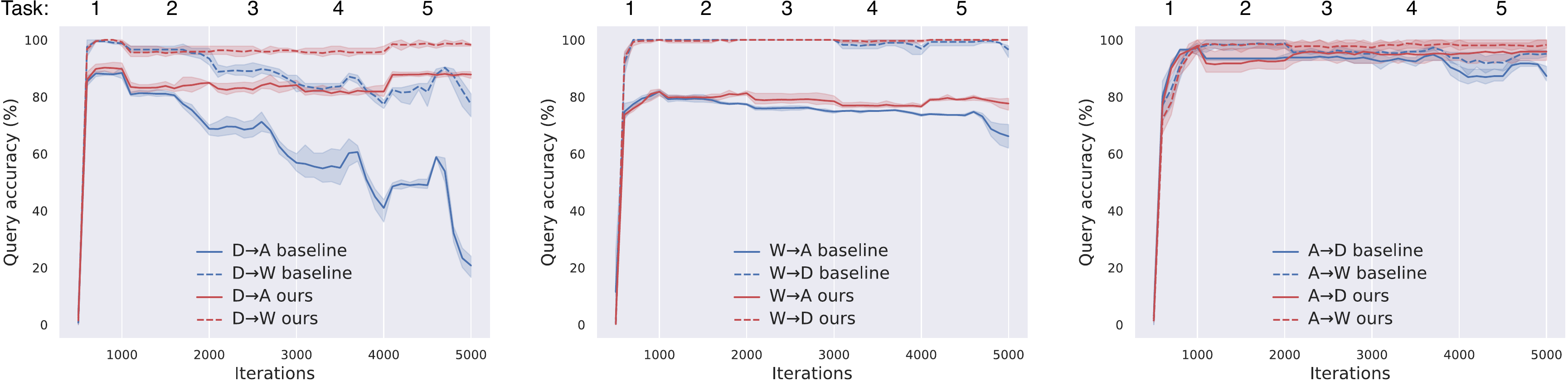}
	\caption{The query accuracy on the first task during sequential training. (Office-31)}
	\label{fig:task1_alldomains}
	\vspace{4pt}
\end{figure*}

\addtolength{\tabcolsep}{3pt}
\begin{table*}[htbp]
    \centering
    \caption{Average query accuracy ($\%$) of all learned tasks on Office-31.}
    \begin{adjustbox}{width=\textwidth}
      \begin{tabular}{c*{8}{c}}
        \toprule
        Method & Setting & A $\rightarrow$ W & D $\rightarrow$ W & W $\rightarrow$ D & A $\rightarrow$ D & D $\rightarrow$ A & W $\rightarrow$ A & \fcolorbox{gray}{lightgray}{Average} \\
        \midrule
        joint training* & \multirow{1}{*}{{i.i.d.}} & 94.5$\pm$0.3 & 98.4$\pm$0.1 & 100.0$\pm$.0 & 93.5$\pm$0.2 & 74.6$\pm$0.3 & 72.2$\pm$0.1 & \fcolorbox{gray}{lightgray}{88.9} \\
        \midrule
        upper bound (task split) & {non i.i.d.} & 83.2$\pm$2.7 & 96.2$\pm$0.5 & 96.3$\pm$2.7 & 84.5$\pm$1.4 & 70.5$\pm$1.0 & 67.3$\pm$2.6 & \fcolorbox{gray}{lightgray}{83.0} \\
        \hdashline
        naive baseline$\dagger$ & \multirow{3}{*}{{non i.i.d.}} & 10.2$\pm$2.3 & 17.3$\pm$1.1 & 17.7$\pm$4.2 & 9.1$\pm$2.5 & 10.5$\pm$0.8 & 11.4$\pm$3.0 & \fcolorbox{gray}{lightgray}{12.7} \\
        optimal baseline$\dagger$ & & 90.8$\pm$1.2 & 91.6$\pm$1.1 & \textbf{97.6$\pm$0.7} & \textbf{86.9$\pm$3.1} & 56.0$\pm$1.7 & 63.3$\pm$3.7 & \fcolorbox{gray}{lightgray}{81.0} \\
        ours & & \textbf{92.0$\pm$2.3} & \textbf{95.6$\pm$0.5} & 95.7$\pm$2.7 & 85.2$\pm$0.6 & \textbf{73.3$\pm$0.4} & \textbf{67.3$\pm$3.1} & \fcolorbox{gray}{lightgray}{\textbf{84.9}} \\
        \bottomrule
        \multicolumn{5}{l}{\footnotesize The results are provided with mean$\pm$std based on three independent experiments.} \\
        \multicolumn{5}{l}{\footnotesize * results cited from MDD~\cite{pmlr-v97-zhang19i}; $\dagger$ defined in Table~\ref{tab_ablation}.} \\
      \end{tabular}
    \end{adjustbox}
    \label{tab_average_31}
\end{table*}
\addtolength{\tabcolsep}{-3pt}

\subsection{Additional results}
\subsubsection{Additional results on Office-31 dataset}
Here, we show additional results that are referenced in Section~\ref{sec_task_drift} on the Office-31 dataset, illustrating the effectiveness of our proposed approach in addressing the task drift. 

Figure~\ref{fig:task1_alldomains} compares the query accuracy of the first task between the baseline and our proposed approach, as the training progresses across tasks (from task 1 to task 5). The proposed approach outperforms the baseline, where the solver is shown to consistently maintain the knowledge on how to solve the first task, regardless of the chosen domains (\eg, D$\rightarrow$A or D$\rightarrow$W). However, the improvement is the most evident in the D$\rightarrow$A setting (Figure~\ref{fig:task1_alldomains}, left).

Table~\ref{tab_average_31} shows the average query accuracy of all learned tasks on Office-31. On average, our approach gives better performance than multiple baselines, and is comparable to the upper bound.

\subsubsection{Example scenario derived from combining Scenario (1) and (2) }
Scenario (1) and (2) are the two most fundamental scenarios that build up most complex ConDA scenarios in real-life. Upon the success of addressing both scenarios in Section~\ref{sec_task_drift} and \ref{sec_domain_drift}, in this subsection, we further show an example scenario that is derived from combining Scenario~(1) and~(2). More specifically, we integrate the domain drift within streams from Scenario (2) into Scenario (1), therefore, the new scenario has both task drift and domain drift within streams, and domain drift across the streams. We show the setup details in Table~\ref{tab_scenario_example_setup}, where we make random combinations of the available tasks and domains.

Table~\ref{tab_scenario_example_31} (Office-31) and Table~\ref{tab_scenario_example_home} (Office-Home) show both the query accuracy of the first task and the average query accuracy of all learned tasks, evaluated on the new example scenario. We compare the proposed approach to both the optimal baseline (defined in Table~\ref{tab_ablation}) and the upper bound. It is noticed that although our proposed approach significantly improves the optimal baseline, there is still a margin between the proposed approach and the upper bound, suggesting further improvement could be explored. We leave the exploration for future work.

\begin{table}[!t]
    \centering
    \caption{Evaluation of the proposed approach on an example scenario by combining Scenario (1) and (2) on Office-31 dataset.} 
    \begin{adjustbox}{width=\columnwidth}
      \begin{tabular}{c*{8}{c}}
        \toprule
        \multirow{2}{*}{Method} & \multicolumn{2}{c}{Query accuracy ($\%$)} \\
        \cmidrule(lr){2-3}
        & The first task only & Average of all learned tasks \\
        \midrule
        upper bound & 88.8$\pm$1.9 & 82.5$\pm$2.4\\
        \hdashline
        optimal baseline$\dagger$ & 15.7$\pm$10.9 & 41.8$\pm$7.1 \\
        ours & \textbf{84.0$\pm$3.1} & \textbf{67.2$\pm$1.7} \\
        \bottomrule
        \multicolumn{3}{l}{\footnotesize The results are provided with mean$\pm$std based on three independent experiments.} \\
        \multicolumn{3}{l}{\footnotesize $\dagger$ defined in Table~\ref{tab_ablation}.} \\
      \end{tabular}
    \end{adjustbox}
    \label{tab_scenario_example_31}
\end{table}

\begin{table}[!t]
    \centering
    \caption{Evaluation of the proposed approach on an example scenario by combining Scenario (1) and (2) on Office-Home dataset.} 
    \begin{adjustbox}{width=\columnwidth}
      \begin{tabular}{c*{8}{c}}
        \toprule
        \multirow{2}{*}{Method} & \multicolumn{2}{c}{Query accuracy ($\%$)} \\
        \cmidrule(lr){2-3}
        & The first task only & Average of all learned tasks \\
        \midrule
        upper bound & 78.4$\pm$0.7 & 81.8$\pm$1.0\\
        \hdashline
        optimal baseline$\dagger$ & 2.7$\pm$0.2 & 15.4$\pm$2.9 \\
        ours & \textbf{74.5$\pm$6.6} & \textbf{61.3$\pm$2.0} \\
        \bottomrule
        \multicolumn{3}{l}{\footnotesize The results are provided with mean$\pm$std based on three independent experiments.} \\
        \multicolumn{3}{l}{\footnotesize $\dagger$ defined in Table~\ref{tab_ablation}.} \\
      \end{tabular}
    \end{adjustbox}
    \label{tab_scenario_example_home}
\end{table}

\begin{table*}[!t]
    \centering
    \caption{Setup of an example scenario derived from combining Scenario (1) and (2).}
      \begin{tabular}{c*{14}{c}}
        \toprule
        Dataset & & \multicolumn{13}{c}{Task \& Domain} \\
        \midrule
        & Task: & 1 & 2 & 3 & 4 & 5 \\
        \cmidrule(lr){2-7}
        \multirow{2}{*}{Office-31} & Support stream: & D & A & W & D & W \\
        & Target stream: & A & W & D & W & A\\
        \midrule
        & Task: & 1 & 2 & 3 & 4 & 5 & 6 & 7 & 8 & 9 & 10 & 11 & 12 & 13 \\
        \cmidrule(lr){2-15}
        \multirow{2}{*}{Office-Home} & Support stream: & Ar & Cl & Pr & Rw & Ar & Rw & Pr & Ar & Pr & Rw & Cl & Ar & Cl \\
        & Target stream: & Cl & Pr & Rw & Ar & Rw & Pr & Cl & Pr & Rw & Cl & Ar & Cl & Rw \\
        \bottomrule
      \end{tabular}
    \label{tab_scenario_example_setup}
\end{table*}

\begin{figure*}[!t]
	\includegraphics[width=\textwidth]{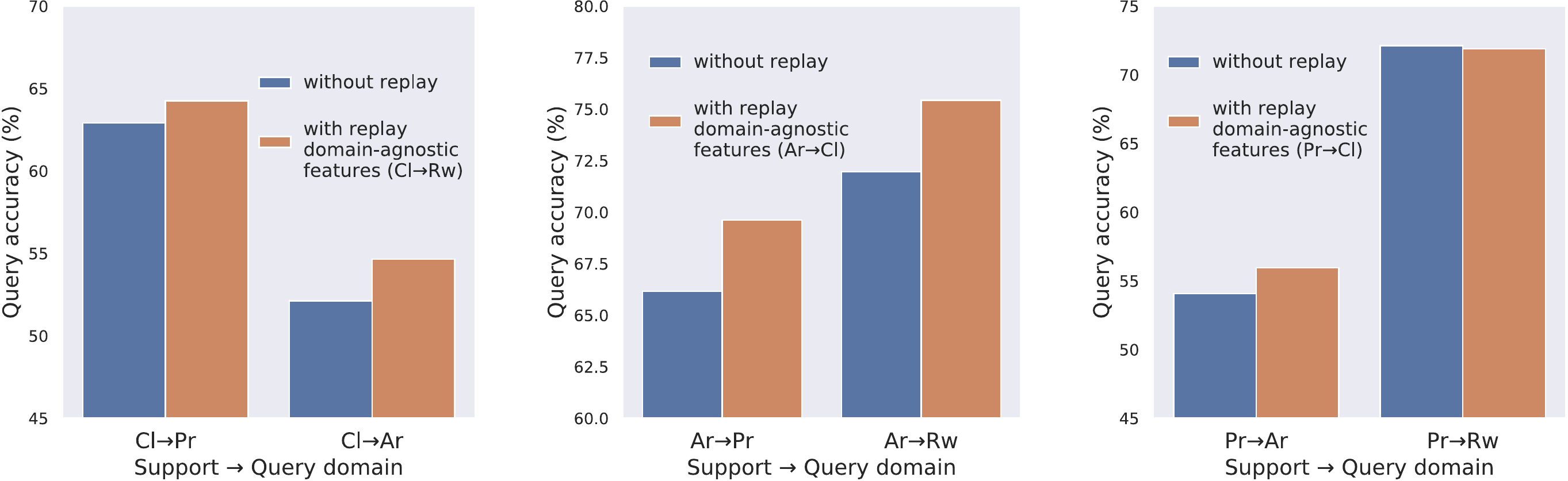}
	\caption{Additional examples of generative feature replay for data augmentation.}
	\label{fig:domains_aug_suppl}
\end{figure*}

\subsubsection{Additional results on generative feature replay for data augmentation}
In this subsection, we provide additional results that is referenced in Section~\ref{sec_exp_data_aug}, where we show the usage of generative feature replay as a data augmentation tool, in addition to addressing catastrophic forgetting. 

Figure~\ref{fig:domains_aug_suppl} shows more examples of generative replay of previously learned features benefiting the adaptation of the current query domain. For example, replaying the domain-agnostic features learned from Cl$\rightarrow$Rw improves both Cl$\rightarrow$Pr and Cl$\rightarrow$Ar (Figure~\ref{fig:domains_aug_suppl}, left), and replaying the features from Ar$\rightarrow$Cl also improves both Ar$\rightarrow$Pr and Ar$\rightarrow$Rw (Figure~\ref{fig:domains_aug_suppl}, middle). In some cases, however, no significant improvement is observed (\eg, from Pr$\rightarrow$Cl to Pr$\rightarrow$ Rw in Figure~\ref{fig:domains_aug_suppl}, right), for the same reason as described in Section~\ref{sec_exp_data_aug}.

\end{document}